\documentclass[12pt]{colt2019} 

\title{Finite-Time Error Bounds For Linear Stochastic Approximation and TD Learning}
\usepackage{times}
\usepackage{cleveref}
\usepackage{amsfonts,wrapfig,bm,color}
\newtheorem{cor}{Corollary}

\coltauthor{%
 \Name{R. Srikant} \Email{rsrikant@illinois.edu}\\
 \addr ECE \& CSL, University of Illinois at Urbana-Champaign
 \AND
 \Name{Lei Ying} \Email{lei.ying.2@asu.edu}\\
 \addr  ECEE, Arizona State University
}

\begin{document}

\maketitle

\begin{abstract}
We consider the dynamics of a linear stochastic approximation algorithm driven by Markovian noise, and derive finite-time bounds on the moments of the error, i.e., deviation of the output of the algorithm from the equilibrium point of an associated ordinary differential equation (ODE). We obtain finite-time bounds on the mean-square error in the case of constant step-size algorithms by considering the drift of an appropriately chosen Lyapunov function. The Lyapunov function can be interpreted either in terms of Stein's method to obtain bounds on steady-state performance or in terms of Lyapunov stability theory for linear ODEs. We also provide a comprehensive treatment of the moments of the square of the 2-norm of the approximation error. Our analysis yields the following results: (i) for a given step-size, we show that the lower-order moments can be made small as a function of the step-size and can be upper-bounded by the moments of a Gaussian random variable; (ii) we show that the higher-order moments beyond a threshold may be infinite in steady-state; and (iii) we characterize the number of samples needed for the finite-time bounds to be of the same order as the steady-state bounds. As a by-product of our analysis, we also solve the open problem of obtaining finite-time bounds for the performance of temporal difference learning algorithms with linear function approximation and a constant step-size, without requiring a projection step or an i.i.d. noise assumption.
\end{abstract}

\section{Introduction}\label{sec: intro}

Reinforcement learning refers to a collection of techniques for solving Markov Decision Problems (MDPs) when the underlying system model is unknown \cite{bertsekas1996neuro,sutton2018reinforcement,bhatnagar2012stochastic,szepesvari2010algorithms,bertsekas2011dynamic}. We consider one of the simplest versions of the problem, where a policy is given and the goal is to compute the value function associated with the policy. Since the state space of the MDP can be very large, it is customary to approximate the value function by a function with far fewer parameters than the state space. Deep neural networks and linear function approximations are the two common approaches for approximating the value function. Here, we are motivated by linear function approximations, and assume that a given function's parameters are updated using temporal difference (TD) learning \cite{sutton1988learning}. One contribution of the paper is to derive finite time bounds on the distance between the parameters estimated by TD and the parameters that minimize the projected Bellman error.

The proof of convergence of TD(0), and more generally TD($\lambda$), was presented in \cite{tsitsiklis1997analysis}. That paper proved asymptotic convergence, but did not study finite-time error bounds. The finite-time performance of the TD algorithm has been studied in \cite{dalal2017finite,lakshminarayanan2018linear}, but it is assumed that the samples required to update the function parameters are i.i.d. We do not require such an assumption in this paper. Bounds on the finite-time mean-square error in the general Markovian case have been derived recently in \cite{bhandari2018finite}. However, in analyzing constant step-size algorithms, they remark that their finite-time results hold only under the assumption that the standard algorithm includes a projection step. We do not require such an assumption here.

We consider a general linear stochastic approximation algorithm of the form considered in \cite{lakshminarayanan2018linear}, but with Markovian noise. Since it is well known that TD algorithms can be viewed as a special case of linear stochastic approximation \cite{tsitsiklis1997analysis, bertsekas1996neuro}, finite-time bounds for linear stochastic approximation can be converted to finite-time bounds on TD algorithms. The first major theme of our analysis is the study of the drift of an appropriately chosen Lyapunov function to obtain an upper bound on the mean-square error. We do this by mimicking the steps in deriving finite time bounds on the square of the 2-norm of the state of a linear ODE. The choice of the Lyapunov function can either be motivated by Stein's method, a method that was originally designed to study central limit theorem approximations, or the stability theory of linear ODEs. When studying the drift, we also condition on the state of the system sufficiently in the past so that the probability distribution has mixed sufficiently to be close to the stationary distribution in an appropriate sense.

The second major theme of our analysis is in extending the drift analysis to study all moments of the parameters estimated by the stochastic approximation algorithm. Here, a key contribution is to show that lower-order moments of the square of the 2-norm of the approximation error can be upper bounded by the moments of a Gaussian distribution, and to show that the moments may not exist in steady-state beyond a threshold order. Our results also imply that the 2-norm of the error of the stochastic approximation algorithm does not have exponentially decaying tails in steady-state. We also discuss the relationship between our results and central limit theorem results obtained in prior literature in the limit where the step-size goes to zero.

The rest of the paper is organized as follows. In Section~\ref{sec: SA}, we consider a version of linear stochastic approximation where the ``noise" is Markovian. The key ideas behind the Lyapunov-Stein approach for studying linear stochastic approximation algorithms with constant step-sizes are presented in this section. The applications to TD(0) and TD($\lambda$) are almost immediate, and these are discussed in Section \ref{sec: TD}. Concluding remarks are provided in Section~\ref{sec: conclusions}.

\section{Linear Stochastic Approximation with Markov Noise}\label{sec: SA}

Consider the following linear stochastic recursion with a constant step size $\epsilon$
\begin{align}
\Theta_{k+1}=\Theta_k+\epsilon \left(A(X_k)\Theta_k+b(X_k)\right),\label{generalform}
\end{align} where $\Theta_k$ is a random vector, $A(X_k)$ is a random matrix and $b(X_k)$ is a random vector, generated according to an underlying Markov chain $\{X_k\}$. Assume the following two limits exist: $$\lim_{k\rightarrow\infty }E[A(X_k)]=\bar{A}\hbox{ and }\lim_{k\rightarrow\infty}E[b(X_k)]=0.$$ The corresponding ODE for this stochastic recursion is
\begin{align}
\dot{\theta}=\bar{A}\theta.\label{eq: ODE1}
\end{align}
The recursion (\ref{generalform}) is called linear stochastic approximation in \cite{lakshminarayanan2018linear}, and we adopt that terminology here.

Assume that $\bar{A}$ is a Hurwitz matrix (i.e., all eigenvalues have strictly negative real parts), and thus, the equilibrium point of the ODE is $0$. We note that, if the steady-state value of $E[b(X_k)]$ is not equal to $0$, then the ODE's equilibrium point will not be $0$. However, by appropriate centering, we can always rewrite both the stochastic recursion and the ODE in the form we consider here. A number of temporal difference algorithms for reinforcement learning, including TD(0), TD($\lambda$) and GTD; and stochastic gradient descent algorithm for linear-square estimation can be written in this form (see the detailed discussion in \cite{lakshminarayanan2018linear}).

We are interested in estimating the deviation from the equilibrium using the metric $E[||\Theta_k||^{2n}]$ for finite $k.$

\subsection{Notation, Assumptions, and Key Ideas}
\label{sec: assump}
Throughout of this paper,  $\|\cdot\|$ denotes the 2-norm for vectors and the induced 2-norm for matrices. We now state assumptions below.
\begin{itemize}
\item {\bf Assumption 1:} $\{X_k\}$ is a Markov chain with state space $\cal S$.  We assume the following two limits exist:
$$\lim_{k\rightarrow\infty }E[A(X_k)]=\bar{A}\hbox{ and }\lim_{k\rightarrow\infty}E[b(X_k)]=0.$$
Define $\tau_\delta\geq 1$ to be the mixing time of $\{X_k\}$ such that
\begin{align}
\|E[b(X_k)|X_0=i]\|\leq \delta \quad \forall i, \forall k\geq \tau_\delta\\
\|\bar{A}-E[A(X_k)|X_0=i]\|\leq \delta  \quad \forall i, \forall k\geq \tau_\delta.
\end{align}
We assume that there exists $K\geq 1$ such as $\tau_\delta\leq K\log\frac{1}{\delta}.$ When considering constant step-size algorithms, we will always choose $\delta=\epsilon,$ and for convenience, we will assume that $\epsilon$ is chosen such that $\epsilon\tau_\epsilon \leq 1/4.$ When the context is clear, we will omit the subscript $\delta$ in $\tau_\delta$ to simplify the notation.

\item {\bf Assumption 2:}    We assume $b_{\max}=\sup_{i\in{\cal S}} \|b(i)\|<\infty$ and $ A_{\max}=\sup_{i\in{\cal S}} \|A(i)\|\leq 1.$ Under this assumption,  it follows that $\|\bar{A}\|\leq A_{\max}\leq 1.$
	
\item {\bf Assumption 3:} All eigenvalues of $\bar{A}$ are assumed to have strictly negative real parts, i.e., $A$ is Hurwitz. This ensures that the ODE is globally, asymptotically stable. This also implies that there exists a symmetric  $P>0,$ which solves the Lyapunov equation
\begin{equation}\label{eq: lyapunov}
-I=\bar{A}^\top P+P\bar{A}.
\end{equation}
Let $\gamma_{\max}$ and $\gamma_{\min}$ denote the largest and smallest eigenvalues of $P,$ respectively.

\end{itemize}

\begin{remark} One part of Assumption 1 states the Markov chain mixes at a geometric rate (i.e., $\tau_\delta\leq K\log\frac{1}{\delta}$), which holds for any finite-state Markov chain which is aperiodic and irreducible \cite{bremaud2013markov}. We assume geometric mixing for notational convenience. Our error bounds, which are in terms of $\epsilon$ and $\tau,$ hold for general mixing rates as long as $\lim\sup_{\epsilon\rightarrow 0}\epsilon \tau_{\epsilon}=0.$
\end{remark}

\begin{remark}
For notational convenience, we assume $A_{\max}\leq 1$ throughout this paper. If $A_{\max} > 1,$ we can normalize $A$ and $b$ as follows, if necessary:
\begin{align*}
{A}(i)\leftarrow \frac{A(i)}{A_{\max}}\hbox{ and } \ {b}(i)\leftarrow \frac{b(i)}{A_{\max}}.
\end{align*} In the context of TD algorithms, this is called feature normalization \cite{bhandari2018finite}.
\end{remark}

Before we present our results, we present the intuition behind them. A standard method to study stochastic recursions is to consider the drift of a Lyapunov function $W$:
$$E[W(\Theta_{k+1})-W(\Theta_k)|H_k],$$
where $H_k$ is some appropriate history which we do not specify yet. Even though we are interested the case where $X_k$ is a Markov chain, it is instructive to get some intuition by considering the case where $X_k$ is i.i.d. Further, any finite-time performance bounds should ideally yield good bounds in steady-state as well, so we will first get some intuition about obtaining good steady-state bounds. When the system is in steady-state, one can use the fact that the unconditioned drift must be equal to zero (subject to the usual caveats about appropriate expectations existing in steady-state):
$$E[W(\Theta_{k+1})-W(\Theta_k)]=0,$$ where $\Theta_k$ is drawn according to a stationary distribution which is assumed to exist. Expanding the left-hand side using Taylor's series, we get
\begin{equation}\label{eq: Taylor's}
E\left[\nabla^\top W(\Theta_k)(\Theta_{k+1}-\Theta_k)+ \frac{1}{2}(\Theta_{k+1}-\Theta_k)^\top\nabla^2 W(\tilde{\Theta})(\Theta_{k+1}-\Theta_k) \right]=0
\end{equation}
for an appropriate $\tilde{\Theta}.$ Now it is interesting to consider how one should choose $W$ so that the solution to the above equation provides a bound on some performance metric of interest. Suppose, we are interested in obtaining a bound on $E[\|\Theta_k\|^2],$ then Stein's method (see \cite{ying2016approximation} and references within) suggests that one should choose $W$ so that
\begin{equation}\label{eq: Stein}
\nabla^\top W(\Theta_k) E\left[\Theta_{k+1}-\Theta_k \left. \right| \Theta_k\right]=-||\Theta_k||^2, \qquad \forall \Theta_k.
\end{equation}
The rationale is that, by substituting (\ref{eq: Stein}) in  (\ref{eq: Taylor's}), we get
$$E\left[||\Theta_k||^2\right]=E\left[\frac{1}{2}(\Theta_{k+1}-\Theta_k)^\top\nabla^2 W(\tilde{\Theta})(\Theta_{k+1}-\Theta_k) )\right],$$ and one can use bounds on the Hessian $\nabla^2 W$ to bound $E[||\Theta_k||^2].$ We do not pursue such bounds here, although one can easily do so based on our analysis later, but we focus on the so-called Stein's equation (\ref{eq: Stein}). Using the fact that $E\left[\Theta_{k+1}-\Theta_k|\Theta_k\right]=\epsilon\bar{A}\Theta_k$ (under the assumption $X_k$ are i.i.d.), we can rewrite (\ref{eq: Stein}) as
$$
\nabla^\top W(\Theta_k) \bar{A}\Theta_k =-||\Theta_k||^2.
$$
To solve for $W,$ we guess that it has a positive-definite quadratic form $$W(\Theta_k)=\Theta_k^\top P\Theta_k$$ and use the fact that Stein's equation must be satisfied for all $\Theta_k$ to obtain
$$\bar{A}^\top P+P\bar{A}=-I.$$
As mentioned in our assumptions, when $\bar{A}$ is Hurwitz, there exists a unique solution satisfying the above equation. Our brief discussions indicates that, for our linear stochastic approximation model, Stein's equation (\ref{eq: Stein}) is equivalent to the Lyapunov equation (\ref{eq: lyapunov}). We note that the function $W$ is the standard Lyapunov function used to study the stability of linear ODEs \cite{chen1998linear}.

We have now argued that a quadratic form for $W$ serves as a good Lyapunov function to obtain bounds on steady-state performance. But we are interested in finite-time bounds.  The key idea behind the derivation of our finite-time performance bounds is very similar to how one would proceed to obtain bounds on $||\theta_t||^2$ for the ODE (\ref{eq: ODE1}).  As is standard in the study of linear ODEs \cite{chen1998linear}, considering the time derivative of the Lyapunov function $W$ along the trajectory of the ODE, we get
\begin{eqnarray*}
	\frac{dW}{dt}&=&\theta^\top P\bar{A}\theta+\theta^\top \bar{A}^\top P\theta =-||\theta||^2 \leq -\frac{1}{\gamma_{\max}} W,
\end{eqnarray*}
where the equality in the second line above is obtained by recalling the Lyapunov equation for $P$ and the inequality is obtained by defining $\gamma_{\max}$ is the largest eigenvalue of $P.$
Thus,
$$W(\theta_t)\leq e^{-\frac{1}{\gamma_{\max}} t} W(\theta_0),$$
which further implies
\begin{align*}
\|\theta_t\|^2\leq \frac{1}{\gamma_{\min}}W(\theta_t)\leq \frac{1}{\gamma_{\min}}e^{-\frac{t}{\gamma_{\max}}} W(\theta_0)\leq \frac{\gamma_{\max}}{\gamma_{\min}}e^{-\frac{t}{\gamma_{\max}}}||\theta_0||^2,
\end{align*}
where $\gamma_{\min}>0$ is the smallest eigenvalue of $P.$
In other words, $\|\theta_t\|^2$ decreases exponentially as $e^{-\frac{t}{\gamma_{\max}}}.$ Our analysis of the stochastic system (Theorem \ref{theorem:ftb}) will show that the mean-square error $E\left[\|\Theta_k\|^2\right]$ approaches its steady-state value as \begin{equation}\label{eq: conv rate}
    \frac{\gamma_{\max}}{\gamma_{\min}}\left(1-\frac{0.9\epsilon}{\gamma_{\max}}\right)^{k-\tau} \|\Theta_0\|^2
\end{equation} when $\epsilon$ is small, which closely resembles the convergence rate of the ODE.

The analysis of the stochastic system is somewhat similar to the analysis of the ODE, except that we will look at the one-time-step drift of the Lyapunov function instead of the time derivative as in the case of the ODE. Additionally, since the bound is motivated by the ODE which is determined only by the steady-state distribution of $A(X_k)$ and $b(X_k),$ we have to study the system after an initial transient period equal to the mixing time defined earlier. This leads to the presence of $\tau$ in (\ref{eq: conv rate}), which does not appear in the corresponding expression for the ODE before that.  We remark that approximating stochastic recursions with ODEs has been extensively studied in the literature. \cite{meerkov1972simplified,meerkov1972simplified-2} is the first papers we are aware of that establishes this connection; comprehensive surveys on this topic can be found in \cite{kushner2003stochastic} and \cite{borkar2009stochastic}. However, to the best of our knowledge, finite-time bounds such as in this paper have not been established before.

In summary, there are three key ideas in the derivation of finite-time bounds on the mean square error of $\Theta_k$: (i) the choice of the Lyapunov function, (ii) the use of the ODE to guide the analysis of the drift of the Lyapunov function, and (iii) an appropriate conditioning of the drift to invoke the mixing properties of the Markov chain $\{X_k\}.$

\subsection{Finite-Time Bounds on the Mean-Square Error}\label{sec: constant}

Before we study the drift of $W,$ we first present a sequence of three lemmas which will be useful for proving the main result later.
The first lemma below essentially states that, since $\Theta_{k}-\Theta_{k-1}$ is of the order of $\epsilon$ for all $k,$ we have $\Theta_\tau-\Theta_0$ is of the order of $\epsilon \tau.$ The subsequent two lemmas provide bounds on terms involving $\Theta_{k+1}-\Theta_{k}$ in terms of $\Theta_k,$ which will be useful later. All of these statements can intuitively inferred from (\ref{generalform}), the proofs presented in the appendix (Appendices \ref{app:lem-tau-0}-\ref{app:lem-linear}) make the intuition precise.

\begin{lemma} The following three inequalities hold when comparing $\Theta_0$ and $\Theta_\tau:$
	\begin{align}
	\|\Theta_{\tau}-\Theta_0\|\leq&2\epsilon \tau  \|\Theta_0\| +2\epsilon \tau b_{\max},\label{lem: boundon1norm}
	\end{align}
	\begin{equation}\|\Theta_\tau-\Theta_0\|\leq 4 \epsilon \tau \|\Theta_\tau\|+4\epsilon \tau b_{\max}\label{lem: boundon2norm-2}\end{equation}
	and
	\begin{equation}\|\Theta_\tau-\Theta_0\|^2\leq 32 \epsilon^2 k^2\|\Theta_\tau\|^2+32\epsilon^2\tau^2 b^2_{\max}.\label{lem: boundon2norm}\end{equation}\hfill{$\square$}
	\
	\label{lem tau-0}
\end{lemma}

\begin{lemma}
	The following inequality holds for any $k\geq 0$
\begin{align*}
\left|(\Theta_{k+1}-\Theta_{k})^\top P(\Theta_{k+1}-\Theta_{k})\right|
\leq 2{\epsilon^2}\gamma_{\max} \left\|\Theta_k\right\|^2+2{\epsilon^2}\gamma_{\max}b^2_{\max}.
\end{align*}\hfill{$\square$}\label{lem: quadratic}
\end{lemma}

\begin{lemma}
	The following inequality holds for all $k\geq \tau:$
\begin{align*}
\left|E\left[\Theta_k^\top P \left(\bar{A}\Theta_k-\frac{1}{\epsilon }\left({\Theta}_{k+1}-{\Theta_k}\right)\right)\middle\vert \Theta_{k-\tau}, X_{k-\tau}\right]\right|
\leq \kappa_1\epsilon \tau E\left[\left.\left\|\Theta_{k}\right\|^2\right|\Theta_{k-\tau}\right]+\kappa_2\epsilon\tau,
\end{align*}
where
\begin{align*}
\kappa_1= & 62\gamma_{\max}(1+b_{\max})\quad\hbox{ and }\quad\kappa_2=  55\gamma_{\max}(1+b_{\max})^3.
\end{align*}\hfill{$\square$}
\label{lem: linear}
\end{lemma}

We are now ready to study the drift of $W(\Theta_k).$
\begin{lemma} For any $k\geq \tau$ and $\epsilon$ such that $\kappa_1\epsilon\tau+{\epsilon \gamma_{\max}}\leq 0.05,$ we have
\begin{align*}
E\left[W({\Theta}_{k+1})\right]\leq \left(1-\frac{0.9\epsilon}{\gamma_{\max}}\right) E\left[W({\Theta_k})\right] +\tilde{\kappa_2}\epsilon^2\tau,
\end{align*}
where
$$\tilde{\kappa}_2=2\left(\kappa_2+\gamma_{\max}b_{\max}^2\right).$$\label{lem: finite}
\end{lemma}
\begin{proof}
Note that for any $k\geq \tau,$ we have
\begin{align*}
&E\left[\left.W({\Theta}_{k+1})-W(\Theta_k)\right|\Theta_{k-\tau}, X_{k-\tau}\right]\\
&= E\left[2\left.\Theta_k^\top P (\Theta_{k+1}-\Theta_k)+(\Theta_{k+1}-\Theta_k)^\top P(\Theta_{k+1}-\Theta_k)\right|\Theta_{k-\tau}, X_{k-\tau}\right]\\
&= E\left[2\left.\Theta_k^\top P (\Theta_{k+1}-\Theta_k-\epsilon \bar{A}\Theta_k)+(\Theta_{k+1}-\Theta_k)^\top P(\Theta_{k+1}-\Theta_k)\right|\Theta_{k-\tau}, X_{k-\tau}\right]\\
&\quad+\epsilon E\left[2\left.\Theta_k^\top P \bar{A}\Theta_k\right|\Theta_{k-\tau}, X_{k-\tau}\right].
\end{align*}
Using the Lyapunov equation, the last term in the previous equation becomes
\begin{equation*}
2\Theta_k^\top P \bar{A}\Theta_k=\Theta^\top_k(\bar{A}^\top P+P\bar{A})\Theta_k=-\|\Theta_k\|^2.
\end{equation*}
Now applying Lemma \ref{lem: quadratic} and Lemma \ref{lem: linear}, we obtain
\begin{align*}
&E\left[\left.W({\Theta}_{k+1})-W(\Theta_k)\right|\Theta_{k-\tau}\right]\\
\leq &-\epsilon E\left[||\Theta_k||^2|\Theta_{k-\tau}\right]+2\epsilon \left(\kappa_1\epsilon\tau+{\epsilon\gamma_{\max}}\right) E\left[||\Theta_k||^2|\Theta_{k-\tau}\right]+2\epsilon^2\tau\left(\kappa_2+\frac{\gamma_{\max}b_{\max}^2}{\tau}\right).
\end{align*}
Given that
$$\kappa_1\epsilon\tau+{\epsilon \gamma_{\max}}\leq 0.05,$$ we have
\begin{align*}
E\left[\left.W({\Theta}_{k+1})-W(\Theta_k)\right|\Theta_{k-\tau}\right]\leq  & -0.9\epsilon E\left[||\Theta_k||^2|\Theta_{k-\tau}\right]+\tilde{\kappa}_2\epsilon^2\tau\\
\leq &-\frac{0.9\epsilon}{\gamma_{\max}} E\left[W(\Theta_k)|\Theta_{k-\tau}\right]+\tilde{\kappa}_2\epsilon^2\tau,
\end{align*}
so the lemma holds.
\end{proof}

\begin{theorem} For any $k\geq \tau$ and $\epsilon$ such that $\kappa_1\epsilon\tau+{\epsilon \gamma_{\max}}\leq 0.05,$ we have the following finite-time bound:
\begin{align}
E\left[\|\Theta_k\|^2\right]\leq \frac{\gamma_{\max}}{\gamma_{\min}}\left(1-\frac{0.9\epsilon}{\gamma_{\max}}\right)^{k-\tau} \left(1.5\|\Theta_0\|+0.5 b_{\max}\right)^2+\frac{\tilde{\kappa}_2\gamma_{\max}}{0.9\gamma_{\min}}\epsilon\tau.\label{thm:ftb}
\end{align}\label{theorem:ftb}
\end{theorem}

\begin{proof} By recursively using the previous lemma, we have
\begin{align*}
E\left[W({\Theta}_{k})\right]\leq a^{k-\tau} E\left[W({\Theta_\tau})\right] +b\frac{1-a^{k-\tau}}{1-a}\leq a^{k-\tau} E\left[W({\Theta_\tau})\right] +b\frac{1}{1-a}
\end{align*} where $a=1-\frac{0.9\epsilon}{\gamma_{\max}}$ and $b=\tilde{\kappa}_2\epsilon^2\tau.$ Furthermore, we have
\begin{align*}
E\left[\|\Theta_k\|^2\right]\leq \frac{1}{\gamma_{\min}}E\left[W({\Theta}_{k})\right]\leq \frac{1}{\gamma_{\min}}a^{k-\tau} E\left[W({\Theta_\tau})\right] +b\frac{1}{\gamma_{\min}(1-a)},
\end{align*} and
\begin{align*}
E\left[W({\Theta_\tau})\right]\leq &\gamma_{\max}E\left[\|\Theta_\tau\|^2\right]\\
\leq &\gamma_{\max}E\left[\left(\|\Theta_\tau-\Theta_0\|+\|\Theta_0\|\right)^2\right]\\
\leq& \gamma_{\max}\left((1+2\epsilon\tau)\|\Theta_0\|+2\epsilon\tau b_{\max}\right)^2,
\end{align*} where the last inequality holds due to  \eqref{lem: boundon1norm}. The theorem holds because $\epsilon\tau\leq \frac{1}{4}.$
\end{proof}

\begin{remark}\label{remark: sec moment}
Using (\ref{thm:ftb}), one can obtain estimates on the number of samples required for the mean-square error to be of the same order as its steady-state value (the second term of the upper bound in \eqref{thm:ftb}). For example, if $k\geq \tau +O(\frac{1}{\epsilon}\log \frac{1}{\epsilon}),$ then it is easy to see that $E[||\Theta_k||^2]$  becomes $O(\epsilon \tau).$ This raises an interesting question: for what values of $k$ is $E[||\Theta_k||^{2n}]$ of the order of $(\epsilon\tau)^n$ for $n >  1$? An answer to this question will show that $\|\Theta_k\|$ is $O(\epsilon \tau)$ in a stronger sense. We answer this question in the next section.
\end{remark}

It is straightforward to extend the analysis in this section to the case of diminishing step sizes, see Appendix \ref{app:dimnishing}. Since the focus of the paper is on constant step-size algorithms, we do not discuss issues such as choosing the stepsizes to optimize the rate of convergence.

\subsection{Finite-Time Bounds on the Higher Moments}
Based on the finite-time bound on the mean-square error, we can further derive the bounds on higher moments of $||\Theta_k||^2$ by induction. In this section, we show that that given a constant step-size $\epsilon,$  for any $n=o\left(\frac{1}{\epsilon\tau}\right),$  the $n$-th moment of $\|\Theta_k\|^2$ can be bounded by the $2n$-th moment of some Gaussian random variable. Further, for sufficiently large $n$ ($n=\omega\left(\frac{1}{\epsilon}\right)$), the higher moments will be $\infty$ in steady state, which may appear to be surprising given standard results on Brownian limits of stochastic recursions in the limit $\epsilon\rightarrow 0.$ We present an explanation and some intuition first before we state our main results for higher moments.

It is standard in the study of certain stochastic recursions to use a higher power of the same Lyapunov function used to obtain lower moment bounds to obtain higher moment bounds, see \cite{eryilmaz2012asymptotically, srikant2013communication}, for example. However, the analysis in these references, which use some equivalent of letting $\epsilon\rightarrow 0$ in this paper, the phenomenon which we observe here does not occur: namely that some higher moments do not exist for each non-zero $\epsilon.$ To get some intuition about why certain higher moments may not exist, consider obtaining a recursion for $E\left[||\Theta_k||^{2n}\right]$ from (\ref{generalform}) for the case where $\Theta_k$ is a scalar and $\{X_k\}$ are i.i.d.; it will be of the form
$$E\left[\Theta_{k+1}^{2n}\right]=E\left[(1+\epsilon A(X_k))^{2n}\right] E\left[\Theta_k^{2n}\right]+ \text{ additional terms}.$$
If $E[(1+\epsilon A(X_k))^{2n})]>1,$ this recursion will blow up to infinity depending on the additional terms above. We will present an example in Appendix \ref{app:example} to show that this can indeed happen.

It is also instructive to compare our results to those obtained from Ornstein-Uhlenbeck (O-U) limits of stochastic recursions such as those studied in \cite{hajek1985stochastic, kushner2003stochastic, borkar2009stochastic}; typically it is shown that $\Theta_k/\sqrt{\epsilon}$ converges to an O-U process in the limit as $\epsilon\rightarrow 0.$ One may be tempted to conclude that it may be possible to obtain tail probabilities of the form $\Pr(||\Theta_k||\geq \sqrt{\epsilon}x)$ using the Gaussian steady-state limit of the O-U process. Our analysis here shows that this is incorrect, in general. In fact, the steady-state distribution is not only not sub-Gaussian, it is not even sub-exponential since the higher moments are all infinity for large $n.$ We remark that the constant step-size analysis of other reinforcement learning algorithms have been considered in \cite{borkar2000ode,beck2012error}, but they do not consider TD learning with a linear function approximation nor higher-moment bounds as we consider here.
Now, we present our main result on higher moments.

\begin{theorem}
Assume the step-size $\epsilon$ satisfies the assumption in Theorem \ref{theorem:ftb}. Then for any integer $n$ such that $\epsilon \tau n\leq \frac{1}{4\sqrt{\gamma_{\min}}}\left(\frac{1}{\gamma_{\min}}+b_{\max}\right),$ there exists $k_n$ such that for any $k\geq k_n,$
\begin{align}
E\left[\left\|\Theta_{k}\right\|^{2n}\right]\leq  (2n-1)!!\left(c\tau\epsilon\right)^{n},
\end{align} where
\begin{equation}
k_n=n\tau+\frac{\tilde{c}}{\epsilon}\left(\log\frac{1}{\epsilon}\right)\sum_{m=1}^n\frac{1}{m},
\label{eq:kn}
\end{equation} and both $c$ and $\tilde{c},$ defined in the appendix, are constants independent of $\epsilon$ and $n.$ \hfill{$\square$}
\label{thm:hm}
\end{theorem}
The proof of this theorem can be found in Appendix \ref{app:thm-hm}. The above result holds for $n=O(1/\epsilon\tau).$  In an example in Appendix \ref{app:example}, it is shown that the $n$th moment of $\|\Theta_\infty\|^2$ for $n=\omega\left(\frac{1}{\epsilon}\right)$ does not exist.

\begin{remark}
Since $n=O(1/\epsilon\tau)$ is sufficient for steady-state moments no higher than $n$ to exist and $\tau=O(\log \frac{1}{\epsilon}),$ it is easy to see that $k=O(\frac{1}{\epsilon}\log^2 \frac{1}{\epsilon})$ is sufficient for the bounds in (\ref{eq:kn}) to hold. This is only off by a logarithmic factor from the sufficient condition for the  bound on the mean-square error to reach a value close to its steady-state; see Remark \ref{remark: sec moment}.
\end{remark}

\section{TD Learning}

\label{sec: TD}

We consider an MDP over a finite state-space, denoted by ${\cal S},$ operating under a fixed stationary policy. We assume that the resulting Markov chain is time-homogeneous, irreducible and aperiodic, and so has a unique stationary distribution $\pi$ to which it converges from any initial probability distribution. We will denote the $i^{\rm th}$ component of the stationary distribution by $\pi(i).$  Since the policy is fixed, we will not use any explicit notation to denote the policy and will consider the resulting Markov chain directly. Let $Z_k\in{\cal S}$ denote the state of the Markov chain in time instant $k.$ We are interested in estimating the value function $V,$ associated with the given policy. The value function at state $i$ is given by
\begin{equation}\label{eq: value}
V(i)=E\left[\left.\sum_{k=0}^\infty \alpha^k c(Z_k)\right|Z_0=i\right],
\end{equation}
where $c(i)$ is the instantaneous reward associated with state $i,$ and $\alpha\in [0,1)$ is the discount factor. It is well known that the value function satisfies
\begin{equation}\label{eq: DP equation}
V(i)=\bar{c}(i)+\alpha \sum_j p_{ij} V(j),
\end{equation}
where $\bar{c}(i)=E[c(i)],$ and $p_{ij}$ is the one-step probability of jumping from state $i$ to state $j.$ If $p_{ij}$ were known, $V$ can be obtained by solving the above equation. Here, our goal is to estimate the value function by observing a trace of the Markov chain $\{Z_0, Z_1, Z_2,\ldots\}.$

Since the size of the state space can be very large,  the goal is to approximate the value function by a linear function of suitably chosen feature vectors as follows:
$$V(i)\approx \phi^\top(i)\theta,$$
where $\theta$ is an unknown weight vector to be estimated and $\phi(i)$ is a feature vector associated with state $i.$ If we denote the size of the state space by $N$ and the dimension of $\theta$ by $d,$ then $d\leq N$ and typically $d<< N.$

\subsection{TD(0)}

Consider the following constant step size version of TD(0) to estimate $\theta:$
\begin{align*}
\Theta_{k+1}=&\Theta_k-\epsilon\phi(Z_k)\left(\phi^\top(Z_k)\Theta_k-c(Z_k)-\alpha \phi^\top(Z_{k+1})\Theta_k\right)\\
=&\Theta_k+\epsilon\left(-\phi(Z_k)\left(\phi^\top(Z_k)-\alpha \phi^\top(Z_{k+1})\right)\Theta_k+ c(Z_k)\phi(Z_k)\right),
\end{align*}
where $\Theta_k$ is the estimate of $\theta$ at time instant $k$ and $\epsilon\in(0,1)$ is a constant step size.

Define $\Phi^\top$ to be matrix whose rows are $\phi^\top(i),$ $D$ to be a diagonal matrix with $D(i,i)=\pi(i),$ where $\pi(i)$ is the stationary distribution of the Markov chain $Z$ in state $i,$ $\Gamma$ to be the transition probability matrix of the Markov chain $Z,$ and $\bar{c}=(\bar{c}(1), \cdots,\bar{c}(i), \cdots)^\top.$ For the case of diminishing step sizes, by verifying the conditions in \cite{benveniste2012adaptive}, it was shown in \cite{tsitsiklis1997analysis} that the algorithm tracks the ODE
\begin{equation}\label{eq: ODE0}
\dot{\theta}=\tilde{A}\theta+\tilde{b},
\end{equation}
and converges to its unique equilibrium point $\theta^*$ under the assumption that $\Phi$ is full rank,
where
\begin{align*}
\tilde{A}=-\Phi D \left(\Phi^\top -\alpha \Gamma \Phi^\top\right)\quad \hbox{ and }\quad \tilde{b}=\Phi D \bar{c}.
\end{align*}
Now by centering the equilibrium point to zero (i.e., $\Theta\leftarrow \Theta-\theta^*$) and defining
\begin{align*}
X_k=&\left(Z_k, Z_{k+1}\right)^\top\\
A(X_k)=&-\phi(Z_k)\left(\phi^\top(Z_k)-\alpha \phi^\top(Z_{k+1})\right)\\
b(X_k)=&c(Z_k)\phi(Z_k)-A(X_k)\theta^*\\
\bar{A}=&\tilde{A}\\
\bar{b}=&0,
\end{align*}
TD(0) can be written as a special case of the general stochastic approximation algorithm form \eqref{generalform}.

To apply the finite-time bound established in Theorem \ref{theorem:ftb}, we next verify the assumptions presented in Section \ref{sec: assump}.
\begin{itemize}
\item {\bf Assumption 1:}
Note that
\begin{align*}
&\|E[b(X_k)|X_0=(Z_0, Z_1)=(z_0, z_1)]\|\\
=&\left\|\sum_i\left(\Pr\left(Z_k=i|Z_0=z_0, Z_1=z_1\right)-\pi_i\right)\left(\bar{c}(i)\phi(i)+\phi(i)\left(\phi^\top(i)-\alpha P_{ij}\phi^\top(j)\right)\theta^*\right)\right\|\\
\leq& b_{\max}\sum_i|\pi_i-\Pr\left(Z_k=i|Z_1=z_1\right)|,
\end{align*} where
$b_{\max}=c_{\max}\phi_{\max}+2\phi^2_{\max}\theta^*,$
and
\begin{align*}
&\|\bar{A}-E[A(X_k)|X_0=(Z_0, Z_1)=(z_0, z_1)]\|\\
=&\left\|\sum_i \left(\Pr\left(Z_k=i|Z_1=z_1\right)-\pi(i)\right)\phi(i)\left(\phi^\top(i)-\alpha P_{ij}\phi^\top(j)\right)\right\|\\
\leq& 2\phi^2_{\max}\sum_i|\pi_i-\Pr\left(Z_k=i|Z_1=z_1\right)|.
\end{align*}
Since $\{Z_k\}$ is a finite state, aperiodic and irreducible Markov chain, it has a geometric mixing rate \cite{bremaud2013markov}, so Assumption 1 holds.

\item {\bf Assumption 2:} To satisfy Assumption 2, we assume $\max_{i\in{\cal S}} \|\phi(i)\|=\phi_{\max} < \infty$ and $\max_{i\in{\cal S}} E[c(i)]=c_{\max}<\infty,$   which implies
\begin{align*}
\|A(X_k)\|=&\left\|-\phi(Z_k)\left(\phi^\top(Z_k)-\alpha \phi^\top(Z_{k+1})\right)\right\|\\
\leq&\|\phi(Z_k)\|^2+\alpha\|\phi(Z_k)\|\|\phi(Z_{k+1})\|\\
\leq& (1+\alpha)\phi_{\max}^2<\infty\\
\|b(X_k)\|\leq& b_{\max}<\infty.
\end{align*} By normalizing the feature vectors, we can have $\phi_{\max}\leq \frac{1}{\sqrt{1+\alpha}},$ which implies that
\begin{align*}
\|A(X_k)\|\leq  (1+\alpha)\phi_{\max}^2\leq 1.
\end{align*} So Assumption 2 holds.

\item {\bf Assumptions 3:} The assumption holds when $\Phi$ is full rank \cite{tsitsiklis1997analysis}. Note that \cite{tsitsiklis1997analysis} proves that $\bar{A}$ is a negative definite matrix (but not necessarily symmetric), not just Hurwitz. In this special case, in addition to the Lyapunov function used in Theorem \ref{theorem:ftb}, one can also use a different Lyapunov function, and follow the rest of the steps in our analysis of general linear stochastic approximation algorithms, to obtain finite-time bounds. We present the analysis for the mean square error in Appendix \ref{sec: neg def}.
\end{itemize}

In summary, the finite-time bound applies when $\Phi$ is full rank, and $\phi_{\max}$ and $c_{\max}$ are bounded.

\subsection{TD($\lambda$)}

In TD($\lambda$), instead of updating the weight vector in the direction of the feature vector of the current state, i.e., $\phi(Z_k),$ one uses the direction of the eligibility trace which is defined to be
$$\varphi_k=(\alpha\lambda) \varphi_{k-1}+\phi(Z_k).$$ In other words,
\begin{align*}
\Theta_{k+1}=\Theta_k-\epsilon\phi(Z_k)\left(\phi^\top(Z_k)\Theta_k-c(Z_k)-\alpha \phi^\top(Z_{k+1})\right)\varphi_k.
\end{align*} Note that $X_k=(Z_k, Z_{k+1}, \varphi_k)$ is a Markov chain.   The algorithm is similar to TD(0) except that now the state-space of the underlying Markov chain is uncountable due to the presence of $\varphi_k.$

The ODE for TD($\lambda$) in the form of $\dot{\theta}=\tilde{A}\theta+\tilde{b}$ (Lemma 6.5 in \cite{bertsekas1996neuro}) has
\begin{align*}
\tilde{A}=\phi^\top D(U-I)\phi\quad\hbox{ and }\quad
\tilde{b}=\phi^\top D q \tilde{c},
\end{align*} where
\begin{align*}
U=(1-\lambda)\sum_{j=0}^\infty \lambda^j(\alpha \Gamma)^{j+1}\quad \hbox{ and }\quad
\tilde{c}=\sum_{j=0}^\infty (\alpha \lambda \Gamma)^j\bar{c}.
\end{align*}

By centering the equilibrium point to zero (i.e. $\Theta\leftarrow \Theta-\theta^*$) and defining
\begin{align*}
A(X_k)=&-\varphi_k\left(\phi^\top(Z_k)-\alpha \phi^\top(Z_{k+1})\right)\\
b(X_k)=&c(Z_k)\varphi_k-A(X_k)\theta^*\\
\bar{A}=&\tilde{A}\\
\bar{b}=&0,
\end{align*} the update of $\Theta_k$ can be written in the form of the general stochastic approximation in Theorem \ref{theorem:ftb}
$$\Theta_{k+1}=\Theta_k+\epsilon\left(A(X_k)\Theta_k+b(X_k)\right).$$
We next verify the assumptions presented in Section \ref{sec: assump}.
\begin{itemize}
\item {\bf Assumption 1:} Given that $\{Z_k\}$ is a finite-state, aperiodic and irreducible Markov chain, geometric mixing  holds according to Lemma 6.7  in \cite{bertsekas1996neuro}.

\item {\bf Assumption 2:} We note that $$\|\varphi_k\|\leq (\alpha\lambda) \|\varphi_{k-1}\|+\|\phi(Z_k)\|\leq (\alpha\lambda) \|\varphi_{k-1}\|+\phi_{\max},$$ which implies that
\begin{align*}
\|\varphi_k\|\leq (\alpha\lambda)^k \|\phi(Z_0)\|+\phi_{\max}\sum_{j=0}^{k-1}(\alpha\lambda)^j\leq \frac{\phi_{\max}}{1-\alpha\lambda}<\infty,
\end{align*}
\begin{align*}
\|A(X_k)\|=&\left\|-\varphi_k\left(\phi^\top(Z_k)-\alpha \phi^\top(Z_{k+1})\right)\right\|\\
\leq&\|\varphi_k\|\left(\|\phi(Z_k)\|+\alpha\|\phi(Z_{k+1})\|\right)\\
\leq& \frac{(1+\alpha)}{1-\alpha\lambda}\phi_{\max}^2<\infty,
\end{align*} and
\begin{align*}
\|b(X_k)\|\leq c_{\max}\|\varphi_k\|+\|A(X_k)\|\theta^*\leq \left(c_{\max}+(1+\alpha)\phi_{\max}\theta^*\right)\frac{\phi_{\max}}{1-\alpha\lambda}.
\end{align*}
Using feature normalization, we can assume $\phi_{\max}\leq \sqrt{\frac{1-\alpha\lambda}{1+\alpha}},$ which implies that
$\|A(X_k)\|\leq  1$ and $\|b(X_k)\|<\infty.$
So Assumption 2 holds.

\item {\bf Assumptions 3:} The assumption holds when $\Phi$ is full rank \cite{tsitsiklis1997analysis}.
\end{itemize}

\section{Conclusions}\label{sec: conclusions}

In this paper, we solve the open problem of obtaining finite-time bounds on the performance of temporal difference learning algorithms using linear function approximation and a constant step-size, without making i.i.d. noise assumptions or requiring a projection step to keep the parameters bounded. Our approach is to consider a more general linear stochastic approximation model and analyze it by studying the drift of a Lyapunov function motivated by Stein's method. Our analysis shows that the moments (up to a certain order) of the square of the 2-norm of the approximation error can be upper-bounded by the moments of a Gaussian random variable; and beyond a certain order, the higher moments become unbounded in steady-state. Our results are also easily extendable to obtain finite-time moment bounds for time-varying step sizes as well.

\appendix

\section{Proof of Lemma \ref{lem tau-0}}
\label{app:lem-tau-0}
	We first have
	\begin{align}
	\|\Theta_{k+1}-\Theta_{k}\| \leq \epsilon \left\|A(X_k)\Theta_k+b(X_k)\right\|\leq\epsilon \left(\|\Theta_k\|+b_{\max}\right),\label{eq:diff}
	\end{align} which implies that
	\begin{align*}
	\|\Theta_{k+1}\| \leq \left(1+\epsilon\right)\|\Theta_k\|+\epsilon b_{\max}.
	\end{align*}
	
	By recursively using the inequality above, we have
	\begin{align*}
	\left\|\Theta_{k}\right\|\leq \left(1+\epsilon\right)^{k}\|\Theta_0\|+\epsilon b_{\max}\sum_{j=0}^{k-1} \left(1+\epsilon\right)^j,
	\end{align*} which is an increasing function in $k.$ Therefore, for any $1\leq k\leq \tau,$ we have
	\begin{align*}
	\left\|\Theta_{k}\right\|\leq & \left(1+\epsilon \right)^\tau \|\Theta_0\|+\epsilon b_{\max}\sum_{j=0}^{\tau-1} \left(1+\epsilon \right)^j\\
	=&\left(1+\epsilon\right)^\tau \|\Theta_0\|+\epsilon b_{\max}\frac{\left(1+\epsilon \right)^\tau-1}{\epsilon}.
	\end{align*}

	Next we want to use the following bound \begin{equation}\left(1+x\right)^\tau \leq 1+2x\tau\label{eq:bound}\end{equation}
	for small $x.$
	Note that  $$\left.\left(1+x\right)^\tau \right|_{x=0} = \left.1+2x\tau\right|_{x=0};$$ and  when $x\leq \frac{\log 2}{\tau-1},$
	$$\frac{\partial}{\partial x}\left(1+x\right)^\tau  =\tau(1+x)^{\tau-1}\leq_{(a)} \tau e^{x(\tau-1)}\leq_{(b)} 2\tau = \frac{\partial}{\partial x} (1+2x\tau)$$ where inequality $(a)$ holds because $\log (1+x)\leq x$ for $x\geq 0,$ and inequality $(b)$ holds when $x\leq \frac{\log 2}{\tau-1}.$
	
	Since we have assumed that $\epsilon\tau\leq 1/4,$ we have $\epsilon\leq\frac{1}{4\tau}\leq \frac{\log 2}{(\tau-1)},$ and so we can apply the bound (\ref{eq:bound}) to obtain  $\forall 1\leq k\leq \tau,$
	\begin{align}
	\left\|\Theta_{k}\right\|\leq \left(1+2\epsilon \tau \right) \|\Theta_0\|+2\epsilon \tau b_{\max}\leq 2 \|\Theta_0\|+2\epsilon \tau b_{\max}, \label{eq:thetabound}
	\end{align} where the last inequality holds because $\epsilon \tau \leq\frac{1}{4}.$

	Now from \eqref{eq:diff} and \eqref{eq:thetabound}, we have \begin{align*}
	\|\Theta_{\tau}-\Theta_0\|\leq &\sum_{k=0}^{\tau-1} \|\Theta_{k+1}-\Theta_k\|\\
	\leq& \epsilon \sum_{k=0}^{\tau-1}\|\Theta_k\|+\epsilon \tau b_{\max}\\
	\leq&   \epsilon \tau \left( 2\|\Theta_0\|+2\epsilon \tau b_{\max}\right)+\epsilon \tau b_{\max}\\
	=&2\epsilon \tau  \|\Theta_0\| +2\epsilon^2\tau^2 b_{\max}+\epsilon \tau b_{\max}\\
	\leq&2\epsilon \tau  \|\Theta_0\| +2\epsilon \tau b_{\max},
	\end{align*} where the last inequality holds because
	\begin{align*}
	2\epsilon^2\tau^2 b_{\max}\leq &\frac{1}{2} \epsilon \tau b_{\max}
	\end{align*} due to the choice of $\epsilon$ which satisfies $\epsilon \tau\leq \frac{1}{4}.$

	From the inequality above, we further have
	\begin{align*}
	\|\Theta_{\tau}-\Theta_0\| \leq&2\epsilon \tau  \|\Theta_0\| +2\epsilon \tau b_{\max}\\
	\leq&2\epsilon \tau  \|\Theta_\tau-\Theta_0\| +2\epsilon \tau  \|\Theta_\tau\| +2\epsilon \tau b_{\max}\\
	\leq&\frac{1}{2} \|\Theta_\tau-\Theta_0\| +2\epsilon \tau  \|\Theta_\tau\| +2\epsilon \tau b_{\max},
	\end{align*} which implies that
	\begin{align*}
	\|\Theta_{\tau}-\Theta_0\| \leq 4\epsilon \tau  \|\Theta_\tau\| +4\epsilon \tau  b_{\max}.
	\end{align*}  This further implies that
	\begin{align*}
	\|\Theta_{\tau}-\Theta_0\|^2
	\leq&32 \epsilon^2\tau^2\|\Theta_\tau\|^2+32 \epsilon^2 \tau^2  b_{\max}^2,
	\end{align*} because $(a+b)^2\leq 2a^2+2b^2.$

\section{Proof of Lemma \ref{lem: quadratic}}
\label{app:lem-quadratic}
Note that
\begin{align*}
&\left|(\Theta_{k+1}-\Theta_{k})^\top P(\Theta_{k+1}-\Theta_{k})\right|\\
\leq &\gamma_{\max}\left\|\Theta_{k+1}-\Theta_{k}\right\|^2\\
\leq& \epsilon^2\gamma_{\max}\left\|A(X_k)\Theta_k +b(X_{k})\right\|^2\\
\leq&\epsilon^2\gamma_{\max}\left(\left\|A(X_k)\Theta_k\right\| +\left\|b(X_{k})\right\|\right)^2\\
\leq& \epsilon^2\gamma_{\max}\left(\left\|\Theta_k\right\| +b_{\max}\right)^2\\
\leq &2\epsilon^2\gamma_{\max}\left(\left\|\Theta_k\right\|^2+b_{\max}^2\right),
\end{align*}
where the second-to-last inequality uses the fact that $\|A(X_k)\|\leq 1.$

\section{Proof of Lemma \ref{lem: linear}}
\label{app:lem-linear}
We prove Lemma \ref{lem: linear} for the case $k=\tau,$ the proof for the general case is essentially identical.  We first note that
\begin{align}
&E\left[\left.\Theta_\tau^\top P \left(\bar{A}\Theta_\tau-\frac{1}{\epsilon }\left({\Theta}_{\tau+1}-{\Theta_\tau}\right)\right)\right|\Theta_0, X_0\right]\nonumber\\
=&E\left[\left.\Theta_\tau^\top P\left(\bar{A}\Theta_\tau-\left(A(X_\tau)\Theta_\tau+b(X_\tau)\right)\right)\right|\Theta_0, X_0\right]\nonumber\\
=&E\left[\left.\Theta_\tau^\top P \left(\bar{A}\Theta_\tau-A(X_\tau)\Theta_\tau\right)\right|\Theta_0, X_0\right]\label{eq:part1}\\
&-E\left[\left.\Theta_\tau^\top P b(X_\tau)\right|\Theta_0, X_0\right].\label{eq:part2}
\end{align}

We first consider \eqref{eq:part1}:
\begin{align*}
&\eqref{eq:part1}\\
=& E\left[\left.\Theta_\tau^\top P \left(\bar{A}\Theta_\tau-A(X_\tau)\Theta_\tau\right)\right|\Theta_0, X_0\right]\\
=& E\left[\left.\Theta_0^\top P \left(\bar{A}\Theta_0-A(X_\tau)\Theta_0\right)\right|\Theta_0, X_0\right]+E\left[\left.(\Theta_\tau-\Theta_0)^\top P \left(\bar{A}-A(X_\tau)\right)\left(\Theta_\tau-\Theta_0\right)\right|\Theta_0, X_0\right]+\\
&E\left[\left.(\Theta_\tau-\Theta_0)^\top P \left(\bar{A}-A(X_\tau)\right)\Theta_0\right|\Theta_0, X_0\right]+E\left[\left.\Theta_0^\top P \left(\bar{A}-A(X_\tau)\right)\left(\Theta_\tau-\Theta_0\right)\right|\Theta_0, X_0\right].
\end{align*}
We next analyze each of the terms above. First we have
\begin{align}
&\left|E\left[\left.\Theta_0^\top P \left(\bar{A}-A(X_\tau)\right)\Theta_0\right|\Theta_0, X_0\right]\right|\nonumber\\
=& \left|\Theta_0^\top P \left(\bar{A}-E\left[\left.A(X_\tau)\right|X_0\right]\right)\Theta_0\right|\nonumber\\
\leq&\|\Theta_0^\top P\|  \left\|\left(\bar{A}-E\left[\left.A(X_\tau)\right| X_0\right]\right)\Theta_0\right\|\nonumber\\
\leq_{(a)} &\epsilon \gamma_{\max} \|\Theta_0\|^2\label{eq:sub1}
\end{align} where inequality (a) holds due to the assumption on the mixing time $\tau.$ Next,
\begin{align}
&\left|E\left[\left.(\Theta_\tau-\Theta_0)^\top P \left(\bar{A}-A(X_\tau)\right)\left(\Theta_\tau-\Theta_0\right)\right|\Theta_0, X_0\right]\right|\nonumber\\
\leq& E\left[\left.\left\|(\Theta_\tau-\Theta_0)^\top P\right\| \left\|\left(\bar{A}-A(X_\tau)\right)\left(\Theta_\tau-\Theta_0\right)\right\|\right|\Theta_0, X_0\right]\nonumber\\
\leq&\gamma_{\max}E\left[\left. (\|\bar{A}\|+\|A(X_\tau)\|)\left\|\Theta_\tau-\Theta_0\right\|^2\right|\Theta_0, X_0\right]\nonumber\\
\leq&2\gamma_{\max}E\left[\left.\left\|\Theta_\tau-\Theta_0\right\|^2\right|\Theta_0, X_0\right].\label{eq:sub2}
\end{align}

Finally,
\begin{align}
&\left|E\left[\left.(\Theta_\tau-\Theta_0)^\top P \left(\bar{A}-A(X_\tau)\right)\Theta_0\right|\Theta_0, X_0\right]\right|+\left|E\left[\left.\Theta_0^\top P \left(\bar{A}-A(X_\tau)\right)\left(\Theta_\tau-\Theta_0\right)\right|\Theta_0, X_0\right]\right|\nonumber\\
\leq &4\gamma_{\max} \|\Theta_0\|E\left[\left.\|\Theta_\tau-\Theta_0\| \right|\Theta_0, X_0\right]\nonumber\\
\leq_{(a)}&8\epsilon\tau \gamma_{\max} \|\Theta_0\|\left(\|\Theta_0\|+b_{\max}\right)\nonumber\\
\leq&8\epsilon\tau \gamma_{\max} \|\Theta_0\|^2 + 8\epsilon\tau \gamma_{\max}\|\Theta_0\| b_{\max}\label{eq:sub3}
\end{align}
 where inequality (a) follows from Lemma \ref{lem tau-0}.

Next we consider \eqref{eq:part2} and use the definition of the mixing time to obtain
\begin{align}
&\left|-E\left[\left.\Theta_\tau^\top P b(X_\tau)\right|\Theta_0, X_0\right]\right|\nonumber\\
=&\left|-E\left[\left.\Theta_0^\top P b(X_\tau)\right|\Theta_0, X_0\right]-E\left[\left.(\Theta^\top_\tau-\Theta_0^\top) P b(X_\tau)\right|\Theta_0, X_0\right]\right|\nonumber\\
\leq&\epsilon\gamma_{\max}\|\Theta_0\|+\gamma_{\max} b_{\max} E\left[\left. \|\Theta_\tau-\Theta_0\| \right| \Theta_0, X_0\right]\nonumber\\
\leq&\epsilon\gamma_{\max}\|\Theta_0\|+2\epsilon\tau\gamma_{\max} b_{\max} \left(\|\Theta_0\|+b_{\max}\right).\label{eq:s2}
\end{align}

By combining the bounds (\ref{eq:sub1})-(\ref{eq:s2}), we have
\begin{align}
&E\left[\left.\Theta_\tau^\top P \left(\bar{A}\Theta_\tau-\frac{1}{\epsilon }\left({\Theta}_{\tau+1}-{\Theta_\tau}\right)\right)\right|\Theta_0, X_0\right]\nonumber\\
\leq& \left(\epsilon\gamma_{\max}+8\epsilon\tau\gamma_{\max}\right)\|\Theta_0\|^2+\left(\epsilon \gamma_{\max}+10\epsilon\tau\gamma_{\max}b_{\max}\right)\|\Theta_0\|+2\epsilon\tau\gamma_{\max}b^2_{\max}\nonumber\\
&+2\gamma_{\max}E\left[\left.\left\|\Theta_\tau-\Theta_0\right\|^2\right|\Theta_0, X_0\right]\nonumber\\
\leq_{(b)}&14 \epsilon\tau\gamma_{\max}(b_{\max}+1)\|\Theta_0\|^2+7\epsilon\tau\gamma_{\max}(b_{\max}+1)^2+2\gamma_{\max}E\left[\left.\left\|\Theta_\tau-\Theta_0\right\|^2\right|\Theta_0, X_0\right]\nonumber\\
\leq_{(c)}&14 \epsilon\tau\gamma_{\max}(b_{\max}+1)E\left[\left. \|\Theta_\tau\|^2 \right| \Theta_0, X_0\right]+7\epsilon\tau\gamma_{\max}(b_{\max}+1)^2\nonumber\\
&+6\gamma_{\max}(b_{\max}+1)E\left[\left.\left\|\Theta_\tau-\Theta_0\right\|^2\right|\Theta_0, X_0\right]\nonumber\\
\leq_{(d)}&14 \epsilon\tau\gamma_{\max}(b_{\max}+1)E\left[\left. \|\Theta_\tau\|^2 \right| \Theta_0, X_0\right]+7\epsilon\tau\gamma_{\max}(b_{\max}+1)^2\nonumber\\
&+6\gamma_{\max}(b_{\max}+1)\left( 32 \epsilon^2 \tau^2E\left[\left.\|\Theta_\tau\|^2\right| \Theta_0, X_0\right]+32\epsilon^2\tau^2 b^2_{\max}\right)\nonumber\\
\leq_{(e)}&62 \epsilon\tau\gamma_{\max}(b_{\max}+1)E\left[\left. \|\Theta_\tau\|^2 \right| \Theta_0, X_0\right]+55\epsilon\tau\gamma_{\max}(b_{\max}+1)^3,\label{eq:sub4}
\end{align}
where inequality (b) holds by noting that $2\|\Theta_0\| \leq 1+ \|\Theta_0\|^2,$ (c) follows from the triangle inequality, (d) follows from Lemma \ref{lem tau-0}, and (e) uses the fact $\epsilon \tau \leq 1/4.$


\section{Proof of Theorem \ref{thm:hm}}
\label{app:thm-hm}
We will use induction to prove this theorem. Suppose the bound holds for $n-1,$ and consider $n.$ To simplify notation, we consider the system starting from $k_{n-1}.$ In other words, in the following analysis, the $k$th iteration is the $(k+k_{n-1})$th iteration of the original system. To simplify our notation, we assume $b_{\max}\geq 1$ without the loss of generality. Since $P$ is a real positive definite matrix, there exists a real positive definite matrix $S$ such that $S^\top S=P,$ and the eigenvalues of $S$ are the square roots of eigenvalues of $P.$ We define
$\Psi=S\Theta,$ so $\Theta^\top P \Theta=\Psi^\top \Psi$ and $\Theta=S^{-1}\Psi.$ Note that $S^\top=S.$

We consider Lyapunov function
\begin{align*}
W_n(\psi)=\left(\psi^\top \psi\right)^n.
\end{align*} The gradient and Hessian of $W_n(\psi)$ are given below:
\begin{align*}
\triangledown W_n(\psi)=&2n\left(\psi^\top \psi\right)^{n-1}\psi\\
\triangledown^2 W_n(\psi)=&4n(n-1)\left(\psi^\top \psi\right)^{n-2}\psi \psi^\top +2n\left(\psi^\top \psi\right)^{n-1}I.
\end{align*}
Taylor's Theorem states $$W_n(\hat{\psi})=W_n(\psi)+(\hat{\psi}-\psi)^\top \triangledown W_n(\psi)+\frac{1}{2}(\hat{\psi}-\psi)^\top \triangledown^2W_n(\tilde{\psi})(\hat{\psi}-\psi),$$ where $\tilde{\psi}=h\psi+(1-h)\hat{\psi}$ for some $h\in[0,1].$ Therefore, we have
\begin{align*}
&\left(\Psi_{k+1}^\top  \Psi_{k+1}\right)^n\\
=&\left(\Psi_{k}^\top  \Psi_{k}\right)^n+(\Psi_{k+1}-\Psi_k)^\top 2n\left(\Psi^\top_k \Psi_k\right)^{n-1} \Psi_k\\
&+(\Psi_{k+1}-\Psi_k)^\top \left(4n(n-1)\left(\tilde{\Psi}^\top \tilde{\Psi}\right)^{n-2}\tilde{\Psi} \tilde{\Psi}^\top +2n\left(\tilde{\Psi}^\top \tilde{\Psi}\right)^{n-1} I  \right)(\Psi_{k+1}-\Psi_k),
\end{align*} where $\tilde{\Psi}=h\Psi_k+(1-h)\Psi_{k+1}$ for some $h\in[0,1],$ which implies that
\begin{align}
&\left(\Psi_{k+1}^\top \Psi_{k+1}\right)^n\nonumber\\
=&\left(\Psi_{k}^\top  \Psi_{k}\right)^n+2(\epsilon S\bar{A}\Theta_k)^\top n\left(\Psi^\top_k \Psi_k\right)^{n-1} \Psi_k\label{eq:hm-1}\\
&+(\Psi_{k+1}-\Psi_k-\epsilon S\bar{A}\Theta_k)^\top 2n\left(\Psi^\top_k \Psi_k\right)^{n-1} \Psi_k\label{eq:hm-2}\\
&+(\Psi_{k+1}-\Psi_k)^\top \left(4n(n-1)\left(\tilde{\Psi}^\top \tilde{\Psi}\right)^{n-2}\tilde{\Psi} \tilde{\Psi}^\top +2n\left(\tilde{\Psi}^\top \tilde{\Psi}\right)^{n-1} I \right)(\Psi_{k+1}-\Psi_k).\label{eq:hm-3}
\end{align}

We will analyze each of the three terms above in the following subsections.
\subsection{Bounding \eqref{eq:hm-1}}
First,  from the Lyapunov equation, we obtain $$2(S\bar{A}S^{-1}\psi)^\top \psi=2\theta^\top \bar{A}^\top P \theta=-\theta^\top \theta=-\psi^\top P^{-1}\psi,$$ which implies that
\begin{align}
\eqref{eq:hm-1}=&\left(\Psi_{k}^\top \Psi_{k}\right)^n-\epsilon n\left(\Psi^\top_k \Psi_k\right)^{n-1} \Psi_k^\top P^{-1}\Psi_k\nonumber\\
=&\left(\Psi_{k}^\top\Psi_{k}\right)^{n-1}\left(\Psi_{k}^\top \Psi_{k}-\epsilon n\Psi_k^\top P^{-1}\Psi_k\right)\nonumber\\
\leq& \left(1-\frac{\epsilon n}{\gamma_{\max}}\right)\left(\Psi_{k}^\top \Psi_{k}\right)^{n}.
\end{align}

\subsection{Bounding \eqref{eq:hm-2}}
Next we have
\begin{align}
\eqref{eq:hm-2}=&-2\epsilon (A(X_k)\Theta_k+b(X_k)- \bar{A}\Theta_k)^\top n\left(\Psi^\top_k \Psi_k\right)^{n-1} S\Psi_k\nonumber\\
=&-2\epsilon (A(X_k)\Theta_0+b(X_k)- \bar{A}\Theta_0)^\top n\left(\Psi^\top_0 \Psi_0\right)^{n-1} S\Psi_0\label{eq:hm-l-1}\\
&-2\epsilon (\left(A(X_k)- \bar{A}\right)\left(\Theta_k-\Theta_0\right))^\top n\left(\Psi^\top_0 \Psi_0\right)^{n-1} S\Psi_0\label{eq:hm-l-2}\\
&-2\epsilon (A(X_k)\Theta_k+b(X_k)- \bar{A}\Theta_k)^\top \left(n\left(\Psi^\top_k \Psi_k\right)^{n-1} S\Psi_k- n\left(\Psi^\top_k \Psi_k\right)^{n-1} S\Psi_0\right)\label{eq:hm-l-3}\\
&-2\epsilon (A(X_k)\Theta_k+b(X_k)- \bar{A}\Theta_k)^\top \left(n\left(\Psi^\top_k \Psi_k\right)^{n-1} S\Psi_0- n\left(\Psi^\top_0 \Psi_0\right)^{n-1} S\Psi_0\right).\label{eq:hm-l-4}
\end{align}

We recall the following inequalities
\begin{align*}
\|\Theta_k-\Theta_0\|\leq& 2\epsilon k \left(\|\Theta_0\|+b_{\max}\right)\\
\left\|A(X_k)\Theta_k+b(X_k)- \bar{A}\Theta_k\right\|\leq& 2 \|\Theta_k\|+b_{\max}.
\end{align*} Also according to Taylor's theorem, we have
\begin{align*}
\left\|\left(\Psi^\top_k \Psi_k\right)^{n-1} -\left(\Psi^\top_0 \Psi_0\right)^{n-1}\right\|\leq & 4\epsilon k(n-1)\sqrt{\gamma_{\max}}\left(\|\Theta_0\|+b_{\max}\right)\|\hat{\Psi}\|^{2n-3},
\end{align*} where $\|\hat\Psi\|=\hat{h}\|\Psi_0\|+(1-\hat{h})\|\Psi_k\|$ for some $\hat{h}\in[0,1].$

We next analyze \eqref{eq:hm-l-1}-\eqref{eq:hm-l-4}. First,
\begin{align*}
\eqref{eq:hm-l-1}=2\epsilon  n\left(\Psi^\top_0 \Psi_0\right)^{n-1} \left(((A(X_k)-\bar{A})\Theta_0)^\top S\Psi_0+b^\top(X_k)S\Psi_0\right).
\end{align*} Based on the mixing assumption, we have
\begin{align*}
E[|\eqref{eq:hm-l-1}|]=&E[E[|\eqref{eq:hm-l-1}||\Theta_0]]\\
\leq &E\left[2\epsilon^2n \left(\sqrt{\frac{\gamma_{\max}}{\gamma_{\min}}}\|\Psi_0\|^{2n}+\sqrt{\gamma_{\max}}\|\Psi_0\|^{2n-1}\right)\right]\\
\leq &E\left[\epsilon^2n \left(\left(2\sqrt{\frac{\gamma_{\max}}{\gamma_{\min}}}+\sqrt{\gamma_{\max}}\right)\|\Psi_0\|^{2n}+\sqrt{\gamma_{\max}}\|\Psi_0\|^{2n-2}\right)\right].
\end{align*}

Next,
\begin{align*}
|\eqref{eq:hm-l-2}|\leq &8\sqrt{\gamma_{\max}}\epsilon^2 kn\left(\left\|\Theta_0\right\|+b_{\max}\right)\|\Psi_0\|^{2n-1}\\
\leq &8\sqrt{\gamma_{\max}}\epsilon^2 kn\left(\frac{1}{\sqrt{\gamma_{\min}}}\left\|\Psi_0\right\|+b_{\max}\right)\|\Psi_0\|^{2n-1}\\
= &8\sqrt{\gamma_{\max}}\epsilon^2 kn\left(\frac{1}{\sqrt{\gamma_{\min}}}\left\|\Psi_0\right\|^{2n}+b_{\max}\|\Psi_0\|^{2n-1}\right)\\
\leq &4\sqrt{\gamma_{\max}}\epsilon^2 kn\left(\left(\frac{2}{\sqrt{\gamma_{\min}}}+b_{\max}\right)\left\|\Psi_0\right\|^{2n}+b_{\max}\|\Psi_0\|^{2n-2}\right),
\end{align*} and
\begin{align*}
|\eqref{eq:hm-l-3}|\leq & 4{\gamma_{\max}}\epsilon^2 k n (2\|\Theta_k\|+b_{\max})^2\|\Psi_k\|^{2n-2}\\
=& 4{\gamma_{\max}}\epsilon^2 k n  \left(4\|\Theta_k\|^{2}+4b_{\max}\|\Theta_k\|+b_{\max}^2\right)\|\Psi_k\|^{2n-2}\\
\leq & 4{\gamma_{\max}}\epsilon^2 k n \left((4+2b_{\max})\|\Theta_k\|^{2}+(b_{\max}^2+2b_{\max})\right)\|\Psi_k\|^{2n-2}\\
\leq & 4{\gamma_{\max}}\epsilon^2 k n \left(\frac{4+2b_{\max}}{\gamma_{\min}}\|\Psi_k\|^{2n}+(b_{\max}^2+2b_{\max})\|\Psi_k\|^{2n-2}\right).
\end{align*}

Finally, we have
\begin{align*}
|\eqref{eq:hm-l-4}|\leq& 8\gamma_{\max}\epsilon^2k n(n-1)(2\|\Theta_k\|+b_{\max}) (\|\Theta_0\|+b_{\max})\|\Psi_0\|\|\hat{\Psi}\|^{2n-3}.
\end{align*}
According to the definition of $\hat\Psi,$ we obtain $$\|\hat{\Psi}\|^{2n-3}\leq \|{\Psi_0}\|^{2n-3}+\|{\Psi_k}\|^{2n-3}.$$  Furthermore, note that $$|x|^a|y|^b\leq |x|^{a+b}+|y|^{a+b}$$ for any $a>0$ and $b>0.$
 Therefore, we have
\begin{align*}
|\eqref{eq:hm-l-4}|\leq& 8\gamma_{\max}\epsilon^2 k n(n-1)\left(\frac{4+3b_{\max}}{\gamma_{\min}}\left(\|\Psi_k\|^{2n}+\|\Psi_0\|^{2n}\right)+\frac{2b^2_{\max}\gamma_{\min}+3b_{\max}}{\gamma_{\min}}\left(\|\Psi_0\|^{2n-2}+\|\Psi_k\|^{2n-2}\right)\right).
\end{align*}

\subsection{Bounding \eqref{eq:hm-3}}
We now consider \eqref{eq:hm-3}, and have
\begin{align*}
\left|\eqref{eq:hm-3}\right|\leq & \gamma_{\max}\epsilon^2\left(\|\Theta_k\|+b_{\max}\right)^2\left(4n^2\|\tilde{\Psi}\|^{2n-2}\right)\\
\leq&4\gamma_{\max}\epsilon^2 n^2\left(2\|\Theta_k\|^2+2b^2_{\max}\right)\left(\|{\Psi}_0\|^{2n-2}+\|{\Psi}_k\|^{2n-2}\right)\\
\leq&4\gamma_{\max}\epsilon^2 n^2\left(\frac{4}{\gamma_{\min}}\|\Psi_k\|^{2n}+\frac{2}{\gamma_{\min}}\|\Psi_0\|^{2n}+2b^2_{\max}\left(\|{\Psi}_0\|^{2n-2}+\|{\Psi}_k\|^{2n-2}\right)\right).
\end{align*}

\subsection{Bounding $E[\|\Psi_0\|^{2n}]$}
We note that
\begin{align*}
\|\Psi_0\|^{2n}-\|\Psi_k\|^{2n}=&\sum_{m=0}^{2n-1}\left(\|\Psi_0\|^{2n-m}\|\Psi_k\|^m-\|\Psi_0\|^{2n-m-1}\|\Psi_k\|^{m+1}\right)\\
=&\sum_{m=0}^{2n-1}\|\Psi_0\|^{2n-m-1}\|\Psi_k\|^m\left(\|\Psi_0\|-\|\Psi_k\|\right)\\
\leq &\sum_{m=0}^{2n-1}2\sqrt{\gamma_{\max}}\epsilon k\|\Psi_0\|^{2n-m-1}\|\Psi_k\|^m(\|\Theta_0\|+b_{\max}).
\end{align*}
Furthermore,
\begin{align*}
&\|\Psi_0\|^{2n-m-1}\|\Psi_k\|^m(\|\Theta_0\|+b_{\max})\\
\leq& \frac{1}{\sqrt{\gamma_{\min}}}\|\Psi_0\|^{2n-m}\|\Psi_k\|^m+b_{\max}\|\Psi_0\|^{2n-m-1}\|\Psi_k\|^m\\
\leq& \left(\frac{1}{\sqrt{\gamma_{\min}}}+b_{\max}\right)\|\Psi_0\|^{2n-m}\|\Psi_k\|^m+b_{\max}\|\Psi_0\|^{2n-m-2}\|\Psi_k\|^m\\
\leq& \left(\frac{1}{\sqrt{\gamma_{\min}}}+b_{\max}\right)\left(\|\Psi_0\|^{2n}+\|\Psi_k\|^{2n}\right)+b_{\max}\left(\|\Psi_0\|^{2n-2}+\|\Psi_k\|^{2n-2}\right).
\end{align*}
Therefore, we have
\begin{align*}
&\|\Psi_0\|^{2n}-\|\Psi_k\|^{2n}\\
\leq &4\sqrt{\gamma_{\max}}\epsilon kn\left(\left(\frac{1}{\sqrt{\gamma_{\min}}}+b_{\max}\right)\left(\|\Psi_0\|^{2n}+\|\Psi_k\|^{2n}\right)+b_{\max}\left(\|\Psi_0\|^{2n-2}+\|\Psi_k\|^{2n-2}\right)\right)\\
=&\epsilon kn \left(\tilde{c}_1(\|\Psi_0\|^{2n}+\|\Psi_k\|^{2n})+\tilde{c}_2\left(\|\Psi_0\|^{2n-2}+\|\Psi_k\|^{2n-2}\right)\right),
\end{align*} where $\tilde{c}_1=4\sqrt{\gamma_{\max}}\left(\frac{1}{\sqrt{\gamma_{\min}}}+b_{\max}\right)$ and $\tilde{c}_2=4\sqrt{\gamma_{\max}}b_{\max},$ which implies that
\begin{align*}
(1-\tilde{c}_1\epsilon k n)\|\Psi_0\|^{2n}-(1+\tilde{c}_1\epsilon kn)\|\Psi_k\|^{2n}\leq \tilde{c}_2\epsilon kn \left(\|\Psi_0\|^{2n-2}+\|\Psi_k\|^{2n-2}\right),
\end{align*} and
\begin{align*}
\|\Psi_0\|^{2n}\leq &\frac{1+\tilde{c}_1\epsilon kn}{1-\tilde{c}_1\epsilon k n}\|\Psi_k\|^{2n}+ \frac{\tilde{c}_2}{1-\tilde{c}_1\epsilon k n}\epsilon kn \left(\|\Psi_0\|^{2n-2}+\|\Psi_k\|^{2n-2}\right).
\end{align*}
Choosing $k=\tau$ and under assumption that $\epsilon k n\leq \frac{1}{2\tilde{c}_1},$ we have
\begin{align*}
\|\Psi_0\|^{2n}\leq& 3\|\Psi_k\|^{2n}+ \frac{\tilde{c}_2}{\tilde{c}_1}\|\Psi_0\|^{2n-2}+\frac{\tilde{c}_2}{\tilde{c}_1}\|\Psi_k\|^{2n-2}.
\end{align*}

\subsection{Higher Moment Bounds}
Choosing $k=\tau,$ from the analysis above, we observe that the bounds we have involve $\|\Psi_k\|^{2n},$ $\|\Psi_0\|^{2n-2}$ and $\|\Psi_\tau\|^{2n-2},$ where  $E\left[\|\Psi\|^{2n-2}\right]\leq (2n-3)!!(c\tau\epsilon)^{n-1}$ based on the induction assumption. Therefore, it is easy to verify that there exist constant $c_1$ and $c_2,$ independent of $\epsilon,$ $\tau$ and $n,$ such that
\begin{align*}
&E\left[\left(\Psi_{\tau+1}^\top \Psi_{\tau+1}\right)^n\right]\\
\leq &\left(1-\frac{\epsilon n}{\gamma_{\max}}\right)E\left[\left(\Psi_{\tau}^\top \Psi_{\tau}\right)^{n}\right]+\epsilon^2 \tau n^2\left(c_1E\left[\left(\Psi_{\tau}^\top \Psi_{\tau}\right)^{n}\right]+c_2 (2n-3)!!(c\tau\epsilon)^{n-1}\right)\\
\leq&\left(1-\frac{0.9\epsilon n}{\gamma_{\max}}\right)E\left[\left(\Psi_{\tau}^\top \Psi_{\tau}\right)^{n}\right]+ \epsilon^2 \tau n c_2 (2n-1)!!(c\tau\epsilon)^{n-1},
\end{align*} where the last inequality holds because $\epsilon=O(\frac{1}{\tau n}).$  The same inequality holds for any $k\geq \tau$ (by conditioning on $\Theta_{k-\tau}$ instead of $\Theta_0$ in the analysis). We therefore have for the original system,
\begin{align*}
E\left[\left(\Psi_{k}^\top \Psi_{k}\right)^n\right]\leq  \left(1-\frac{0.9\epsilon n}{\gamma_{\max}}\right)^{k-k_{n-1}-\tau}E\left[\left(\Psi_{k_{n-1}+\tau}^\top \Psi_{k_{n-1}+\tau}\right)^{n}\right]+ \frac{10c_2\gamma_{\max}}{9}\tau \epsilon  (2n-1)!!(c\tau\epsilon)^{n-1}.
\end{align*}

Since
\begin{align*}
\|\Psi_{k_{n-1}+\tau}\|\leq &\|\Psi_0\|+\|\Psi_{k_{n-1}+\tau}-\Psi_0\| \\
\leq&\sqrt{\gamma_{\max}}\left(\|\Theta_0\|+2\epsilon (k_{n-1}+\tau)(\|\Theta_0\|+b_{\max})\right)\\
\leq&3\sqrt{\gamma_{\max}}\epsilon k_{n-1}\left(\|\Theta_0\|+b_{\max}\right),
\end{align*} where the last inequality holds because $k_{n-1}=\omega\left(\frac{1}{\epsilon}\right)$ and $\tau=\log\frac{1}{\epsilon},$
we have $$E\left[\left(\Psi_{k_{n-1}+\tau}^\top \Psi_{k_{n-1}+\tau}\right)^{n}\right]\leq \left(3\sqrt{\gamma_{\max}}\epsilon k_{n-1}\left(\|\Theta_0\|+b_{\max}\right)\right)^2E\left[\left(\Psi_{k_{n-1}+\tau}^\top \Psi_{k_{n-1}+\tau}\right)^{n-1}\right],$$ which implies that
\begin{align*}
&E\left[\left(\Psi_{k+1}^\top \Psi_{k+1}\right)^n\right]\\
\leq &\left(1-\frac{0.9\epsilon n}{\gamma_{\max}}\right)^{k-k_{n-1}-\tau}9{\gamma_{\max}}\epsilon^2 k^2_{n-1}\left(\|\Theta_0\|+b_{\max}\right)^2(2n-3)!!(c\tau\epsilon)^{n-1}+ \frac{10c_2\gamma_{\max}}{9}\tau \epsilon  (2n-1)!!(c\tau\epsilon)^{n-1}.
\end{align*}
Therefore, we conclude that for
\begin{equation}
k\geq k_n=k_{n-1}+\tau+\frac{\log\frac{\epsilon k^2_{n-1}\left(\|\Theta_0\|+b_{\max}\right)^2}{2c_2\tau n}}{-\log\left(1-\frac{0.9\epsilon n}{\gamma_{\max}}\right)},
\label{eq:kn1}
\end{equation} we have the following bound
\begin{align}
E\left[\left(\Psi_{k+1}^\top \Psi_{k+1}\right)^n\right]\leq 11c_2\gamma_{\max}\tau \epsilon  (2n-1)!!(c\tau\epsilon)^{n-1}.
\end{align} So the theorem holds by defining $c=11\max\{c_2\gamma_{\max},\frac{\tilde{\kappa}_2\gamma_{\max}}{\gamma_{\min}}\}$ and by noting that
$$\frac{\log\frac{\epsilon k^2_{n-1}\left(\|\Theta_0\|+b_{\max}\right)^2}{2c_2\tau n}}{-\log\left(1-\frac{0.9\epsilon n}{\gamma_{\max}}\right)}=O\left(\frac{1}{\epsilon n}\log\frac{1}{\epsilon}\right).$$

\section{Example showing that higher moments may not exist}
\label{app:example}
We consider the following example:
$$\Theta_{k+1}=\Theta_k+\epsilon\left(A(X_k)\Theta_k+b(X_k)\right),$$ where $X_k\in\{-1, 1\}$ are independent Bernoulli random variables (across $k$) such that  $$\Pr\left(X_k=-1\right)=\Pr\left(X_k=1\right)=0.5.$$ Furthermore, we define $A(-1)=-2,$ $b(-1)=-1,$ and $A(1)=b(1)=1.$ Therefore, the ODE is $$\dot{\theta}=-\theta.$$

Consider the $2n$-th moment of $\Theta$ at steady state. Suppose the $2n$-th moment exists and assume the system is at steady state at time $0.$ We then have
$$E\left[\Theta_1^{2n}\right]=E\left[\Theta_0^{2n}\right],$$ i.e.
$$E\left[E\left[\left.\Theta_1^{2n}-\Theta_0^{2n}\right|\Theta_0\right]\right]=0.$$ From Taylor's Theorem, we have
\begin{align*}
&\Theta_1^{2n}-\Theta_0^{2n}\\
=&2n\Theta_0^{2n-1}\left(\Theta_1-\Theta_0\right)+n(2n-1)\tilde{\Theta}^{2n-2}\left(\Theta_1-\Theta_0\right)^2\\
=&2\epsilon n\Theta_0^{2n-1}\left(A(X_0)\Theta_0+b(X_0)\right)+\epsilon^2 n(2n-1)\tilde{\Theta}^{2n-2}\left(A(X_0)\Theta_0+b(X_0)\right)^2
\end{align*} where $\tilde{\Theta}=h\Theta_0+(1-h)\Theta_1$ for some $h\in[0,1].$ Since $X_0$ is independent of $\Theta_0,$ we have
\begin{align*}
E\left[\Theta_1^{2n}-\Theta_0^{2n}|\Theta_0\right]
=&-2\epsilon n\Theta_0^{2n}+\epsilon^2 n(2n-1) E\left[\left.\tilde{\Theta}^{2n-2}\left(A(X_0)\Theta_0+b(X_0)\right)^2\right|\Theta_0\right].
\end{align*}
Note that if $\Theta_0> 0,$ then when $A(X_0)=b(X_0)=1,$ which occurs with probability 0.5, we have $\tilde{\Theta}>\Theta_0,$ and
\begin{align*}
\tilde{\Theta}^{2n-2}\left(A(X_0)\Theta_0+b(X_0)\right)^2\geq \Theta_0^{2n-2}\left(\Theta_0+1\right)^2\geq \Theta_0^{2n}.
\end{align*}
If $\Theta_0< 0,$ then when $A(X_0)=-2$ and $b(X_0)=-1,$ which occurs with probability 0.5, we have $-\tilde{\Theta}>-\Theta_0,$ and
\begin{align*}
\tilde{\Theta}^{2n-2}\left(A(X_0)\Theta_0+b(X_0)\right)^2\geq \Theta_0^{2n-2}\left(-2\Theta_0-1\right)^2\geq 4\Theta_0^{2n}.
\end{align*}
Therefore, we can conclude
\begin{align*}
E\left[\left.\tilde{\Theta}^{2n-2}\left(A(X_0)\Theta_0+b(X_0)\right)^2\right|\Theta_0\right]\geq \frac{1}{2}\Theta_0^{2n},
\end{align*} which implies that
\begin{align*}
E\left[\left.\Theta_1^{2n}-\Theta_0^{2n}\right|\Theta_0\right] =&\left(-2\epsilon n+\frac{1}{2}\epsilon^2 n(2n-1)\right)\Theta_0^{2n}\\
\geq &\left(\epsilon (n-1)-2\right)\epsilon n\Theta_0^2.
\end{align*} Therefore, when $n=\omega(1/\epsilon),$ we have
\begin{align*}
0=E\left[\Theta_1^{2n}-\Theta_0^{2n}\right] \geq \left(\epsilon (n-1)-2\right)\epsilon nE\left[\Theta_0^2\right]>0,
\end{align*} which leads to the contradiction and proves that the $2n$-th moment does not exist when $n=\omega(1/\epsilon).$

\section{Diminishing Step Sizes}
\label{app:dimnishing}

Considering the stochastic recursion
\begin{align}
\Theta_{k+1}=\Theta_k+\epsilon_k \left(A(X_k)\Theta_k+b(X_k)\right),\label{generalform-diminishing}
\end{align} where $\epsilon_k$ is the step-size used at iteration $k.$ The step-size satisfies the following assumption: $\epsilon_k$ is a nonincreasing sequence and there exists $k^*>0$ and $\kappa_s>0$ such that for any $k\geq k^*,$ $k-\tau_{\epsilon_k}\geq 0$ and $\frac{\epsilon_{k-\tau_{\epsilon_k}}}{\epsilon_k}\leq \kappa_s.$

\begin{theorem} Define $\hat{k}$ to be the smallest integer such that $\hat{k}\geq k^*,$ $\hat{k}\epsilon_0\leq 1/4,$ and $$\kappa_1\kappa_s\epsilon_{\hat{k}}\tau_{\epsilon_{\hat k}}+{\gamma_{\max}\epsilon_{\hat k}}\leq 0.05.$$ Then for any $k\geq \hat{k},$ we have
\begin{align*}
E\left[\|{\Theta}_{k}\|\right]\leq \frac{\gamma_{\max}}{\gamma_{\min}}\left(1.5\|\Theta_0\|+0.5 b_{\max}\right)^2\left(\prod_{j=\hat{k}}^{k-1} a_j\right) + \check{\kappa}_2\sum_{j=\hat{k}}^{k-1}b_j\left(\prod_{l=j+1}^{k-1}a_l\right),
\end{align*} where $a_j=1-\frac{0.9\epsilon_j}{\gamma_{\max}},$ $b_j= \epsilon_j^2\tau_{\epsilon_j},$ and
$\check{\kappa}_2=2{\kappa_2}\kappa_s+2\gamma_{\max} b_{\max}^2.$
 \hfill{$\square$} \label{thm: diminishing}
\end{theorem}

To prove the above theorem, we again use drift analysis with the same Lyapunov function $$W(\Theta_k)=\Theta_k^\top P\Theta_k.$$  We first present the modified versions  of Lemmas \ref{lem: quadratic} and \ref{lem: linear} for diminishing step size.
\begin{lemma}
\begin{align*}
\left|E\left[\left.(\Theta_{k}-\Theta_{k-1})^\top P(\Theta_{k}-\Theta_{k-1})\right|{\Theta}_0\right]\right|
\leq 2{\epsilon}^2_k\gamma_{\max} E\left[\left.\left\|\Theta_k\right\|^2\right|\Theta_0\right]+2{\epsilon^2_k}\gamma_{\max}b^2_{\max}.
\end{align*}
\hfill{$\square$}\label{lem: quadratic-d}
\end{lemma}
The proof of this lemma is identical to that of Lemma \ref{lem: quadratic}.

\begin{lemma} For any $k$ such that $k\geq k^*$ and $k-\tau_{\epsilon_k}\geq 0,$  the following bound holds:
\begin{align*}
\left|E\left[\left.\Theta_k^\top P \left(\bar{A}\Theta_k-\frac{1}{\epsilon_k }\left({\Theta}_{k+1}-{\Theta_k}\right)\right)\right|\Theta_{k-\tau_{\epsilon_k}}\right]\right|
\leq \kappa_1 \kappa_s \epsilon_k \tau_{\epsilon_k} E\left[\left\|\Theta_k\right\|^2|\Theta_{k-\tau_{\epsilon_k}}\right]+\kappa_2\kappa_s\epsilon_k \tau_{\epsilon_k}.
\end{align*}
\label{lem: linear-d}
\end{lemma}
\begin{proof}
Define $l=k-\tau_{\epsilon_k},$ which is a nonnegative number because $k-\tau_{\epsilon_k}\geq 0$ according to the assumption. By following the proof of Lemma \ref{lem: linear} where we replace $\tau$ with $\tau_{\epsilon_k}$ and $\epsilon$ with $\epsilon_l,$ and by simplifying the constants $\kappa_1$ and $\kappa_2$ based on the fact that $\epsilon_l\tau_{\epsilon_k}\leq \frac{1}{4},$ we can obtain
\begin{align*}
\left|E\left[\left.\Theta_k^\top P \left(\bar{A}\Theta_k-\frac{1}{\epsilon_k }\left({\Theta}_{k+1}-{\Theta_k}\right)\right)\right|\Theta_l\right]\right|\leq \kappa_1 \epsilon_l \tau_{\epsilon_k} E\left[\left\|\Theta_k\right\|^2|\Theta_{k-\tau_{\epsilon_k}}\right]+\kappa_2\epsilon_l \tau_{\epsilon_k}.
\end{align*} Therefore, the lemma holds because $\epsilon_l\leq \kappa_s \epsilon_k$ for $k>k^*.$
\end{proof}

Following the proof of Lemma \ref{lem: finite}, we have that for $k\geq \hat{k},$
\begin{align*}
&E\left[\left.W({\Theta}_{k+1})-W(\Theta_k)\right|\Theta_{k-\tau_{\epsilon_k}}\right]\\
=& -\epsilon_k E\left[||\Theta_k||^2|\Theta_{k-\tau_{\epsilon_k}}\right]-\\
&E\left[\left.\epsilon_k \nabla^\top W(\Theta_k)(\bar{A}\Theta_k)+(\Theta_k)^\top P (\Theta_{k+1}-\Theta_k)\right|\Theta_{k-\tau_{\epsilon_k}}\right]+\\
&E\left[\left.\frac{1}{2}(\Theta_{k+1}-\Theta_k)^\top P (\Theta_{k+1}-\Theta_k)\right|\Theta_{k-\tau_{\epsilon_k}}\right].
\end{align*}
By applying the previous two lemmas, we obtain
\begin{align*}
&E\left[\left.W({\Theta}_{k+1})-W(\Theta_k)\right|\Theta_{k-\tau_{\epsilon_k}}\right]\\
\leq &-\epsilon_k E\left[||\Theta_k||^2|\Theta_{k-\tau_{\epsilon_k}}\right]+2\epsilon_k \left(\kappa_1\kappa_s\epsilon_k\tau_{\epsilon_k}+{\epsilon_k\gamma_{\max}}\right) E\left[||\Theta_k||^2|\Theta_{k-\tau_{\epsilon_k}}\right]+\\
&2{\kappa_2}\kappa_s\epsilon_k^2\tau_{\epsilon_k}+2\gamma_{\max}b_{\max}^2\epsilon_k^2.\end{align*}
Under the assumption
$$\kappa_1\kappa_s\epsilon_k\tau_{\epsilon_k}+{\gamma_{\max}\epsilon_k}\leq 0.05,$$ we have
\begin{align*}
E\left[\left.W({\Theta}_{k+1})-W(\Theta_k)\right|\Theta_{k-\tau_{\epsilon_k}}\right]\leq  & -0.9\epsilon_k E\left[||\Theta_k||^2|\Theta_{k-\tau_{\epsilon_k}}\right]+2{\kappa_2}\kappa_s\epsilon_k^2\tau_{\epsilon_k}+2\gamma_{\max}b_{\max}^2\epsilon_k^2\\
\leq &-\frac{0.9\epsilon}{\gamma_{\max}} E\left[W(\Theta_k)|\Theta_{k-\tau}\right]+2{\kappa_2}\kappa_s\epsilon_k^2\tau_{\epsilon_k}+2\gamma_{\max}b_{\max}^2\epsilon_k^2\\
\leq &-\frac{0.9\epsilon}{\gamma_{\max}} E\left[W(\Theta_k)|\Theta_{k-\tau}\right]+ 2\left(\kappa_2\kappa_s+2\gamma_{\max}b_{\max}^2\right)\epsilon_k^2\tau_{\epsilon_k},
\end{align*}
which implies
\begin{align*}
E\left[W({\Theta}_{k+1})\right]\leq \left(1-\frac{0.9\epsilon_k}{\gamma_{\max}}\right) E\left[W({\Theta_k})\right] +\check{\kappa_2}\epsilon_k^2\tau_{\epsilon_k},
\end{align*} where
$$\check{\kappa}_2=2{\kappa_2}\kappa_s+2\gamma_{\max} b_{\max}^2.$$

By recursively using the previous inequality, we have for any $k$ that satisfies the assumptions of theorem, the following inequality holds
\begin{align*}
E\left[W({\Theta}_{k})\right]\leq& \left(\prod_{j=\hat{k}}^{k-1} a_j\right) E\left[W({\Theta_{\hat{k}}})\right] +\tilde{\kappa}_2K\sum_{j=\hat{k}}^{k-1}b_j\left(\prod_{l=j+1}^{k-1}a_l\right)\\
\leq& \gamma_{\max}\left((1+2\epsilon_0\hat{k})\|\Theta_0\|+2\epsilon_0\hat{k} b_{\max}\right)^2\left(\prod_{j=\hat{k}}^{k-1} a_j\right) + \check{\kappa}_2\sum_{j=\hat{k}}^{k-1}b_j\left(\prod_{l=j+1}^{k-1}a_l\right).
\end{align*} The theorem holds because $W(\Theta_k)\geq \gamma_{\min}\|\Theta_k\|.$

\section{Negative Definite $\bar{A}$}\label{sec: neg def}
If $\bar{A}$ is not only Hurwitz but also negative definite (but not necessarily symmetric), then we can use a simple quadratic Lyapunov function $V(\Theta_k)=\|\Theta_k\|^2.$  Considering constant step $\tau,$ we have any $k\geq \tau,$ we have
\begin{align*}
&E\left[\left.\|{\Theta}_{k+1}\|^2-\|\Theta_k\|^2\right|\Theta_{k-\tau}\right]\\
=&E\left[\left.\Theta_k^\top (\Theta_{k+1}-\Theta_k)+\frac{1}{2}\|\Theta_{k+1}-\Theta_k\|^2\right|\Theta_{k-\tau}\right]\\
=&  E\left[\left.\epsilon \Theta_k^\top (A(X_k)\Theta_k+b(X_k))+\frac{1}{2}\|\Theta_{k+1}-\Theta_k\|^2\right|\Theta_{k-\tau}\right]\\
=_{(a)}& \epsilon E\left[\left.\Theta_k^\top \bar{A}(\Theta_k)\right|\Theta_{k-\tau}\right]+ \epsilon E\left[\left. \Theta_k^\top (A(X_k)\Theta_k+b(X_k)-\bar{A}\Theta_k)\right|\Theta_{k-\tau}\right]+\\
&E\left[\left.\frac{1}{2}\|\Theta_{k+1}-\Theta_k\|^2\right|\Theta_{k-\tau}\right]\\
\leq & \epsilon \lambda_{\min}E\left[\left.\|\Theta_k\|^2\right|\Theta_{k-\tau}\right]+ \epsilon E\left[\left.\Theta_k^\top (A(X_k)\Theta_k+b(X_k)-\bar{A}\Theta_k)\right|\Theta_{k-\tau}\right]+\\
&E\left[\left.\frac{1}{2}\|\Theta_{k+1}-\Theta_k\|^2\right|\Theta_{k-\tau}\right]
\end{align*} where equality (a) holds because and $\lambda_{\min}<0$ is the largest eigenvalue of the negative definite matrix $\bar{A}.$

We can apply Lemma \ref{lem: quadratic} and Lemma \ref{lem: linear} with $P=I,$ (i.e., $\gamma_{\min}=\gamma_{\max}=1,$ to bound the second and third terms above. It is easy to verify that we have the following finite-time bounds for constant step size and diminishing step size for using this simple quadratic Lyapunov function when $\bar{A}$ is negative definite.
\begin{cor}
For any $k\geq \tau$ and $\kappa_1\epsilon\tau+{\epsilon \gamma_{\max}}\leq 0.05,$ we have the following finite-time bound:
\begin{align}
E\left[\|\Theta_k\|^2\right]\leq \left(1-0.9\lambda_{\min}\epsilon\right)^{k-\tau} \left(1.5\|\Theta_0\|+0.5 b_{\max}\right)^2+\frac{\tilde{\kappa}_2}{0.9}\epsilon\tau.
\end{align}\hfill{$\square$}\label{cor:ftb}
\end{cor}

\bibliographystyle{}
\bibliography{stein}

\end{document}